\documentclass[11pt,letterpaper]{article}

\usepackage[margin=1in]{geometry}
\usepackage[T1]{fontenc}
\usepackage[utf8]{inputenc}
\usepackage{lmodern}
\usepackage{microtype}

\usepackage{amsmath,amssymb,amsfonts,amsthm,mathtools,bm}
\usepackage{array,textcomp}
\theoremstyle{plain}
\newtheorem{theorem}{Theorem}

\newtheorem{proposition}{Proposition}

\theoremstyle{definition}

\theoremstyle{remark}

\usepackage{algorithm}
\usepackage{algorithmic}
\usepackage{graphicx}
\usepackage{booktabs}
\usepackage{tikz}
\usetikzlibrary{shapes,arrows,3d,positioning,calc,spy}
\usepackage{pgfplots}
\usepgfplotslibrary{groupplots}
\pgfplotsset{compat=1.18}
\usepackage[font=small,labelfont=bf]{caption}

\usepackage[authoryear,round]{natbib}
\usepackage{xurl}
\usepackage[hidelinks]{hyperref}
\hypersetup{
  pdftitle={Robust conditional dimension reduction for dissimilarity data},
  pdfauthor={Xiao Ling and Anh T. Bui},
  pdfsubject={Robust conditional multidimensional scaling},
  pdfkeywords={distance scaling, ISOMAP, multidimensional scaling, SMACOF,
    dissimilarity, robust, Sammon mapping}
}

\newlength{\PlotBaseWidth}
\newcommand{\dashbigarc}[5]{%
  \draw[dotted, ->]
  plot[
    domain=#4:#5,
    samples=150,
    variable=\t
  ]
  (axis cs:{#1 + #3*cos(\t)}, {#2 + #3*sin(\t)});
}

\allowdisplaybreaks[2]
\title{\Large\bfseries Robust conditional dimension reduction\\
for dissimilarity data}
\author{%
  Xiao Ling$^{1}$ \qquad Anh T. Bui$^{2}$\\[0.7em]
  {\small $^{1}$Department of Mathematics, Auburn University at Montgomery}\\
  {\small Montgomery, AL 36117, USA}\\
  {\small \href{mailto:xling@aum.edu}{\texttt{xling@aum.edu}}}\\[0.5em]
  {\small $^{2}$Department of Mathematics and Statistics, Virginia Commonwealth University}\\
  {\small Richmond, VA 23284, USA}\\
  {\small \href{mailto:buiat2@vcu.edu}{\texttt{buiat2@vcu.edu}}}%
}
\date{}

\begin{document}
\maketitle

\begin{abstract}
Conditional dimension reduction (cDR) learns low-dimensional latent coordinates while accounting for observed covariates that represent known sources of variation in the data. Conditional Multidimensional Scaling (cMDS) is a cDR technique that works directly with dissimilarity data. Its standard squared-stress formulation, however, is sensitive to contaminated dissimilarity, since outliers can dominate the objective and distort the learned configuration. We proposed Robust Conditional Multidimensional Scaling (rcMDS) by replacing the squared-stress criterion with a Fair M-estimation objective. We developed a reweighted conditional SMACOF algorithm to optimize this objective. The proposed algorithm admits computationally tractable updates, and its stabilized objective values decrease monotonically and converge to a finite limit. Experiments on synthetic and real data show that the proposed method is substantially less sensitive to corrupted data than cMDS and produces more reliable low-dimensional representations.

\end{abstract}

\section{Introduction}
In many applications, practitioners prefer to work with pairwise dissimilarity measures of subjects of interest, rather than their coordinate map, for several reasons. First, it is easier to obtain pairwise dissimilarities than coordinates. For example, psychology participants can more easily judge how similar two stimuli are than identifying their spatial representation~\citep{jaworska2009review}; sports analysts can directly quantify the dissimilarities between teams from round-by-round match results~\citep{machado2017multidimensional} instead of looking at all possible teams' features. Second, the Euclidean distance is implied when using coordinate maps, and this can be too restrictive. In contrast, pairwise dissimilarity measures give practitioners more flexibility in comparing subjects. For instance, ecologists define dissimilarity indices from species abundance data~\citep{oksanen2013vegan}; materials scientists define dissimilarity measures for images of random heterogeneous media, whose random nature is problematic for the Euclidean distance~\citep{Bui27052021}. 

Regardless of the reason, practitioners ultimately want to obtain a low-dimensional coordinate representation for the subjects from their pairwise dissimilarities for subsequent analyses. Multidimensional scaling (MDS)~\citep{borg2025multidimensional} is one of the earliest fundamental techniques developed to achieve this goal. In essence, MDS produces a low-dimensional configuration that preserves the observed dissimilarities. In mathematical terms, the problem is as follows: given a pairwise dissimilarity matrix $\Delta\in\mathbb{R}^{n\times n}$ for $n$ subjects, we seek an configuration $\mathbf{X}\in\mathbb{R}^{n\times p}$ whose pairwise Euclidean distance $\mathbf{D}\in\mathbb{R}^{n\times n}$ approximates $\Delta$ as closely as possible, with $p<n$. We can understand $p$ as the number of underlying features that control the dissimilarities among the $n$ subjects.

In many applications, some underlying features are known. For example, surface measurements of cylinders machined by lathe turning~\citep{colosimo2011analyzing} can differ systematically in part due to preset cutting depth and turning speed. Accurate mapping and localization in wireless sensor networks are central to applications ranging from surveillance and navigation to wireless communications; accordingly, various MDS-based methods have been proposed to improve sensor-node localization accuracy~\citep{dai2023mds,zou2023multidimensional,luo2025hierarchical,saeed2019state,zhou2022mocloc,kashniyal2017wireless,corbalan2023self}. More recently, \citet{10371783} introduced conditional MDS, which extends classical MDS by incorporating known features as prior information during the configuration stage. This framework recovers coordinates associated only with the unknown features while accounting for the influence of the known ones, thereby improving the quality of the low-dimensional representation and facilitating interpretation.

However, standard conditional MDS is based on a squared-stress criterion and is therefore sensitive to contaminated dissimilarities. Because squared residuals penalize large deviations quadratically, a small number of corrupted pairwise dissimilarities can dominate the objective and distort the recovered configuration. This sensitivity is particularly problematic for proximity data, where an error in a single dissimilarity may affect the low-dimensional representation. In classical MDS, for example, double centering converts the squared dissimilarity matrix into a centered inner product matrix by subtracting row and column means. Thus, a perturbation in one dissimilarity entry can be propagated across many
entries of the transformed matrix~\citep{deng2025robust,cayton2006robust}. These limitations motivate the use of robust residual criteria. Although the least-absolute ($\ell_1$) loss provides a natural alternative by penalizing large residuals linearly, its direct iterative reweighted formulation leads to singular weights near zero residuals. In this paper, we propose robust conditional multidimensional scaling (rcMDS), which replaces the squared-stress loss with a scaled Fair M-estimation criterion. The resulting objective retains the robustness of absolute-residual fitting while yielding finite residual-dependent weights, and we develop a Robust Conditional SMACOF algorithm (rcSMACOF) to optimize it while preserving the separation between known and unknown sources of variation.

Section~\ref{sec:preliminaries} presents the preliminaries and problem formulation, including the notation, assumptions, NP-hardness result, and objective-value convergence analysis for the proposed algorithm. In Section \ref{sec:rcsmacof} we describe the robust conditional SMACOF (\textit{Scaling by Majorizing a Complicated Function}) approach based on conditional SMACOF proposed in \citet{10371783}. Section \ref{sec:experiment} presents experimental results that validate our approach. 

\section{Preliminaries and Problem Formulation} \label{sec:preliminaries}
Assume that \(\Delta=(\delta_{ij})_{i,j=1}^n\in[0,\infty)^{n\times n}\) is a symmetric dissimilarity matrix with \(\delta_{ii}=0\), and that \(\mathbf W=(w_{ij})_{i,j=1}^n\in[0,\infty)^{n\times n}\) is a symmetric nonnegative weight matrix. We assume that \(\mathbf W\) is irreducible, meaning that the \(n\) objects cannot be partitioned into two nonempty groups for which all intergroup weights are zero. If the raw pairwise dissimilarities are asymmetric, we symmetrize them by replacing each pair with the average of the two directed dissimilarities. Let
\[
\mathcal E=\{(i,j):1\le i<j\le n,\; w_{ij}>0\}
\]
denote the active weighted pair set. 

Suppose each object is associated with an observed covariate vector \(\mathbf v_i\in\mathbb R^q\), and let
\[
\mathbf V=[\mathbf v_1^\top;\ldots;\mathbf v_n^\top]\in\mathbb R^{n\times q}
\]
denote the matrix of prior information. We assume that \(\mathbf V\) contains \(q\) linearly independent pairwise difference vectors \(\mathbf v_i-\mathbf v_j\), \((i,j)\in\mathcal E\).

Given $\Delta$ and $\mathbf{V}$, conditional multidimensional scaling (cMDS, \citealp{10371783}) aims to find a latent configuration \(\mathbf U=[\mathbf u_1^\top;\ldots;\mathbf u_n^\top]\in\mathbb R^{n\times p}\), with \(\mathbf u_i\in\mathbb R^p\). Because the intrinsic dimension of the observed data is at most $n-1$, it is required that $n>p+q$. To make the unknown features $\mathbf{u}$ and the known features $\mathbf{v}$ compatible in a single manifold coordinate system, an affine transformation $\mathbf{B}^\top \mathbf{v}$, where \(\mathbf B\in\mathbb R^{q\times q}\), is needed. Thus, $[\mathbf{U}, \mathbf{VB}]\in \mathbb R^{n\times (p+q)}$ contain all the manifold feature values of the $n$ objects. 

For brevity, we write \(k=(i,j)\in\mathcal E\), and use \(\delta_k=\delta_{ij}\), \(w_k=w_{ij}\), and \(r_k(\mathbf X)=r_{ij}(\mathbf X)\). Let \(\mathbf X=(\mathbf U,\mathbf B)\), cMDS estimates \(\mathbf X\) by minimizing the \emph{conditional stress}
\begin{equation}
\begin{aligned}
f(\mathbf X) &= \sum_k w_k\,\rho\bigl(r_k(\mathbf X)\bigr),  && \rho(t)=t^2,\\
r_k(\mathbf X) &= \delta_k-d_k(\mathbf X), &&d_k(\mathbf X) = \sqrt{
        \lVert \mathbf u_i-\mathbf u_j\rVert_2^2 +  \left\lVert
        \mathbf B^\top(\mathbf v_i-\mathbf v_j)
        \right\rVert_2^2
    }.
\end{aligned}
\label{eq:f1}
\end{equation}
To minimize \eqref{eq:f1}, Bui~\citeyearpar{10371783} developed a Conditional SMACOF procedure by extending the majorization framework of SMACOF~\citep{de2009multidimensional}. However, \(\rho(r)\) remains sensitive to unusual residuals because it is based on squared conditional stress. A natural robust alternative is to choose $\rho(r) = |r|$ criterion, which penalizes large deviations linearly and can be optimized through an Iteratively Reweighted Least Squares (IRLS) paradigm~\citep{holland1977robust}. Related \(\ell_1\)-based formulations have also been developed for robust
dimension reduction. In particular, \cite{ling26} estimate a sparse, outlier-insensitive one-dimensional subspace. In contrast, the present work operates directly on pairwise dissimilarities and incorporates observed covariates through the conditional-distance model.

\subsection{\texorpdfstring{$\ell_1$}{l1}-norm cMDS} 
A robust low-dimensional configuration can be obtained by solving the following unconstrained optimization problem:
\begin{equation}
\min_{\mathbf{X}}
\sigma(\mathbf{X})
=
\min_{\mathbf{X}}
\sum_{k}
w_{k}
\left|
r_{k}(\mathbf{X})
\right|.
\label{eq:lcmds}
\end{equation}
We refer to \(\sigma(\cdot)\) as the \textit{conditional strife}, based on the terminology introduced in \citet{Critchley_1989}. It is nonlinear and nonconvex in the joint variables \((\mathbf U,\mathbf B)\), since the fitted interpoint distances \(d_{k}(\mathbf U,\mathbf B)\) depend nonlinearly on both the latent coordinates and the covariate transformation~\citep{dattorro2010convex}. Computing a globally optimal low-dimensional configuration is NP-hard in general~\citep{cayton2006robust}.  

To handle the computational difficulty of \(\ell_1\)-cMDS, we first develop a new majorization-based algorithm that converts the nonsmooth conditional strife into a sequence of fixed-weight conditional stress problems. The key idea is to construct, at each iteration, a quadratic surrogate of the conditional strife by updating residual-dependent weights. The \textit{conditional strife} can be majorized by an iteratively reweighted quadratic function~\citep{heiser1988multidimensional}. Specifically, in the SMACOF majorization, the current iterate is used as a
support point around which the surrogate function is constructed.  Let
\(\mathbf Y=(\widetilde{\mathbf U},\widetilde{\mathbf B})\) denote this support
point, where \(\widetilde{\mathbf U}\) and \(\widetilde{\mathbf B}\) are the
current values of the latent configuration and covariate transformation,
respectively. Recall that we write \(k=(i,j)\) for an active pair
\((i,j)\in\mathcal E\), and following this notation, if \(r_k(\mathbf Y)\neq0\) for all active pairs, then using the standard quadratic majorization of the absolute loss
\citep{heiser1988multidimensional}, we majorize the conditional strife
\(\sigma(\mathbf X)\) by
\[
Q(\mathbf X|\mathbf Y)
=
\frac12
\sum_k
q_k(\mathbf Y)
\left\{
r_k(\mathbf X)^2+r_k(\mathbf Y)^2
\right\},
\]
where
\[
q_k(\mathbf Y)=\frac{w_k}{|r_k(\mathbf Y)|}.
\]
Minimizing \(Q(\mathbf X|\mathbf Y)\) is equivalent to minimizing the fixed-weight conditional
stress
\[
\sum_k q_k(\mathbf Y) r_k(\mathbf X)^2.
\]
This fixed-weight subproblem has the same form as the standard conditional MDS stress and can therefore be solved using Conditional SMACOF~\citep{10371783}. This observation connects the proposed construction to the standard iteratively reweighted least-squares (IRLS) framework. In IRLS, a nonsquared residual loss is handled by repeatedly forming a weighted least-squares surrogate; the weights are computed from the residuals at the current iterate and then held fixed during the next least-squares update.

The choice of residual loss determines the resulting influence and weight functions. The least-absolute loss has bounded influence, but its IRLS weight becomes singular at zero residuals, which creates numerical and theoretical difficulties for the reweighted algorithms~\citep{kalczynski2025further,mankovich2022flag,aftab2015convergence,peng2023convergence}. We therefore use the Fair loss as stable robust surrogate for the \textit{conditional strife} criterion. The Fair loss is a classical M-estimator~\citep{de2021review} and has been used in robust regression and IRLS computation. It behaves approximately like the $\ell_1$-norm for large anomalies, but smooths the cusp at the origin (zero residuals). Fig.~\ref{fig:rho_psi_weight} summarizes this relationship for the squared, least-absolute, and Fair losses.    

\begin{figure}[htpb]
\centering
\begin{tikzpicture}
\pgfmathsetmacro{\etafair}{0.35}

\begin{groupplot}[
    group style={
        group size=3 by 3,
        horizontal sep=0.045\textwidth,
        vertical sep=0.8cm,
    },
    width=0.27\textwidth,
    height=0.19\textwidth,
    scale only axis,
    axis lines=middle,
    axis line style={->},
    ytick=\empty,
    xtick=\empty,
    samples=180,
    domain=-3:3,
    title style={font=\small},
    label style={font=\small},
    clip=false
]
\nextgroupplot[
    title={$r^2$},
    ylabel={\(\rho\)},
    ymin=0,
    ymax=4.8
]
\addplot[black, line width=0.45pt, no marks] {0.5*x^2};

\nextgroupplot[
    title={$|r|$},
    ymin=0,
    ymax=4.8
]
\addplot[black, line width=0.45pt, no marks] {abs(x)};

\nextgroupplot[
    title={Fair},
    ymin=0,
    ymax=4.8
]
\addplot[black, line width=0.45pt, no marks]
    {abs(x)-\etafair*ln(1+abs(x)/\etafair)};

\nextgroupplot[
    ylabel={\(\psi\)},
    ymin=-3.2,
    ymax=3.2
]
\addplot[black, line width=0.45pt, no marks] {x};

\nextgroupplot[
    ymin=-1.4,
    ymax=1.4
]
\addplot[
    black,
    line width=0.45pt,
    no marks,
    domain=-3:-0.03
] {-1};
\addplot[
    black,
    line width=0.45pt,
    no marks,
    domain=0.03:3
] {1};
\addplot[
    black,
    only marks,
    mark=o,
    mark size=0.8pt
] coordinates {(0,-1) (0,1)};

\nextgroupplot[
    ymin=-1.4,
    ymax=1.4
]
\addplot[black, line width=0.45pt, no marks]
    {x/(abs(x)+\etafair)};

\nextgroupplot[
    ylabel={\(w\)},
    ymin=0,
    ymax=4.2
]
\addplot[black, line width=0.45pt, no marks] {1};

\nextgroupplot[
    ymin=0,
    ymax=4.2
]
\addplot[
    black,
    line width=0.45pt,
    no marks,
    domain=-3:-0.25
] {1/abs(x)};
\addplot[
    black,
    line width=0.45pt,
    no marks,
    domain=0.25:3
] {1/abs(x)};
\draw[dashed, line width=0.30pt]
    (axis cs:0,0) -- (axis cs:0,4.0);
\node[font=\small, anchor=south]
    at (axis cs:0,4.0) {\(\infty\)};

\nextgroupplot[
    ymin=0,
    ymax=4.2
]
\addplot[black, line width=0.45pt, no marks]
    {1/(abs(x)+\etafair)};

\end{groupplot}
\end{tikzpicture}

\caption{Loss (\(\rho\)), influence (\(\psi\)), and weight
(\(w\)) functions for least-squares, least-absolute, and Fair
M-estimators.}
\label{fig:rho_psi_weight}
\end{figure}
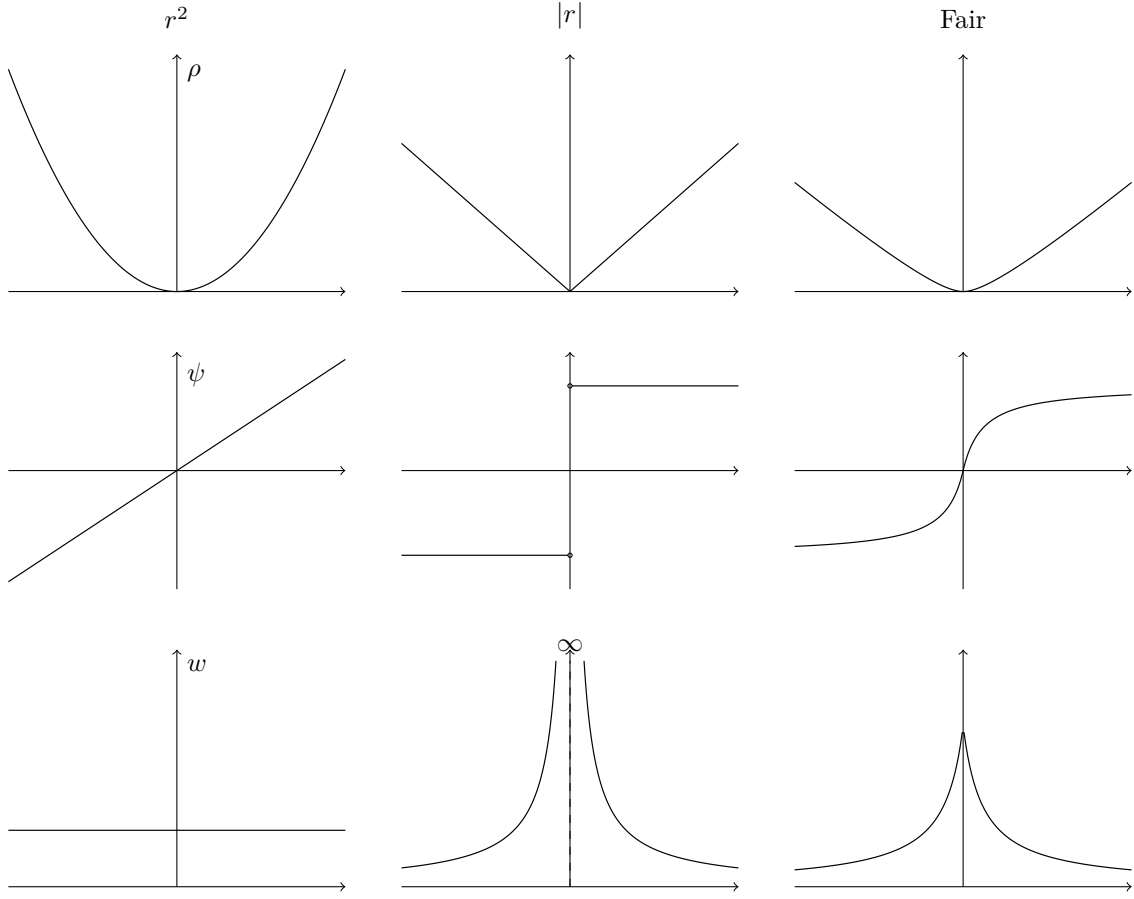

This leads to a numerically stable iteratively reweighted framework. Specifically, we formulate the robust conditional MDS (rcMDS) problem as follows:
\begin{equation}
\begin{aligned}
\min_{\mathbf X}\sigma_\eta(\mathbf X)
&=
\min_{\mathbf X}
\sum_{k}
w_{k}\rho\left(r_{k}(\mathbf X)\right),\\
\rho(r_{k}(\mathbf X))
&=
|r_{k}(\mathbf X)|-\eta\ln\left(1+\frac{|r_{k}(\mathbf X)|}{\eta}\right),
\quad \eta>0.
\end{aligned}
\label{eq:rcMDS}
\end{equation}

\begin{theorem}[NP-hardness of rcMDS]
\label{thm:rcmds_nphard}
For any fixed \(\eta>0\), globally minimizing the rcMDS objective \eqref{eq:rcMDS} is NP-hard. The result remains valid in the restricted
setting \(p=1\), for any fixed \(q\geq1\), with
\(\mathbf v_i=\mathbf0\in\mathbb R^q\) for all \(i\) and
\(w_{ij}=1\) for all \(1\leq i<j\leq n\).
\end{theorem}

\begin{proof}
Section~5.1.1 of \cite{cayton2006robust} establishes NP-hardness of the
one-dimensional embedding problem
\[
\min_{x_1,\ldots,x_n\in\mathbb R}
\sum_{i,j}
g\!\left(f(\delta_{ij})-f(|x_i-x_j|)\right)
\]
when \(f\) and \(g\) are symmetric, nondecreasing in the absolute values
of their arguments, Lipschitz on \([0,1]\), and satisfy, for some
\(\lambda_L>0\), $|h(x)-h(y)|\geq\lambda_L |x-y|\max\{x,y\}, x,y\in[0,1], h\in\{f,g\}.$ 
Set
\[
f(t)=|t|,
\qquad
g(t)=\rho_\eta(t)
=
|t|-\eta\ln\left(1+\frac{|t|}{\eta}\right).
\]
Both functions are symmetric and nondecreasing in absolute value. Moreover, both \(f\) and \(g\) are \(1\)-Lipschitz on \([0,1]\), so the upper-Lipschitz condition holds with \(\lambda_U=1\). For \(0\leq y\leq x\leq1\), we have 
\[
|g(x)-g(y)|
\geq
\frac{1}{2(1+\eta)}
|x-y|\max\{x,y\}
\]
and,
\[
|f(x)-f(y)|
=
|x-y|
\geq
\frac{1}{2(1+\eta)}
|x-y|\max\{x,y\}.
\]
Thus the conditions of \cite{cayton2006robust} hold with $\lambda_U=1, \lambda_L=\frac{1}{2(1+\eta)}>0,$ and the corresponding one-dimensional Fair-loss embedding problem is NP-hard.

Given such an instance, construct an rcMDS instance using the same
dissimilarities and set $p=1, \mathbf v_i=\mathbf0, w_{ij}=1.$ 
Then the covariate contribution vanishes and
\[
d_{ij}(\mathbf U,\mathbf B)
=
\sqrt{|u_i-u_j|^2+
\|\mathbf B^\top(\mathbf v_i-\mathbf v_j)\|^2}
=
|u_i-u_j|.
\]
Because \(\delta_{ij}\) and \(|u_i-u_j|\) are nonnegative,$g\!\left(f(\delta_{ij})-f(|u_i-u_j|)\right)=\rho_\eta\!\left(\delta_{ij}-d_{ij}(\mathbf U,\mathbf B)\right).$ 
The restricted rcMDS objective is the unordered-pair form of the NP-hard one-dimensional Fair-loss embedding objective. Since the reduction is polynomial, restricted rcMDS, and therefore general rcMDS is NP-hard.
\end{proof}

\begin{theorem}[Quadratic majorization of the rcMDS objective]
\label{thm:stabilized_majorization}
Fix \(\mathbf Y=(\widetilde{\mathbf U},\widetilde{\mathbf B})\) and define \(q_k(\mathbf Y)=w_k/(|r_k(\mathbf Y)|+\eta)\). Then \(Q_\eta(\mathbf X\mid\mathbf Y):=\sigma_\eta(\mathbf Y)+\frac{1}{2}\sum_k q_k(\mathbf Y)\bigl[r_k(\mathbf X)^2-r_k(\mathbf Y)^2\bigr]\) satisfies \(\sigma_\eta(\mathbf X)\le Q_\eta(\mathbf X\mid\mathbf Y)\) for all \(\mathbf X\), with \(Q_\eta(\mathbf Y\mid\mathbf Y)=\sigma_\eta(\mathbf Y)\). Hence, \(Q_\eta(\cdot\mid\mathbf Y)\) majorizes \(\sigma_\eta\) at \(\mathbf Y\). Moreover, minimizing \(Q_\eta(\mathbf X\mid\mathbf Y)\) over \(\mathbf X\) is equivalent to minimizing the fixed-weight conditional stress \(\sum_k q_k(\mathbf Y)r_k(\mathbf X)^2\).
\end{theorem}

\begin{proof}
Let $\rho_\eta(t)=|t|-\eta\ln\left(1+\frac{|t|}{\eta}\right)$ and define $g_\eta(s)=\rho_\eta(\sqrt{s}), s\ge0.$ For \(s>0\), $g_\eta'(s)=1/[2(\sqrt{s}+\eta)],$
with \(g_\eta'(0^+)=1/(2\eta)\). Since \(g_\eta'\) is nonincreasing,
\(g_\eta\) is concave on \([0,\infty)\). Therefore,  $g_\eta(s)\le g_\eta(t)+g_\eta'(t)(s-t),s,t\ge0.$
Taking
\(s=r_k(\mathbf X)^2\) and \(t=r_k(\mathbf Y)^2\) gives
\[
\rho_\eta(r_k(\mathbf X))
\le
\rho_\eta(r_k(\mathbf Y))
+
\frac{
r_k(\mathbf X)^2-r_k(\mathbf Y)^2
}{
2(|r_k(\mathbf Y)|+\eta)
}.
\]
Multiplying by \(w_k\), summing over \(k\), and using the definition of
\(q_k(\mathbf Y)\) yields $\sigma_\eta(\mathbf X) \le Q_\eta(\mathbf X\mid\mathbf Y).$
Equality holds at \(\mathbf X=\mathbf Y\). Finally, all terms in
\(Q_\eta(\mathbf X\mid\mathbf Y)\) other than
\(\frac12\sum_k q_k(\mathbf Y)r_k(\mathbf X)^2\) are independent of
\(\mathbf X\).
\end{proof}

\section{Robust Conditional SMACOF Algorithm}
\label{sec:rcsmacof}
Theorem~\ref{thm:stabilized_majorization} places the proposed rcMDS objective within the IRLS framework: at a fixed support point, the Fair objective $\sigma_\eta$ is majorized by $Q_\eta$, and minimizing $Q_\eta$ is equivalent to minimizing a fixed weight condtional stress. Consequently, each outer iteration reduces to minimizing a fixed-weight conditional stress problem.

\begin{theorem}[Descent and objective value convergence of rcSMACOF]
\label{thm:rcMDS_convergence}
At outer iteration \(v\), define \(q_k^{[v]}=w_k/(|r_k(\mathbf X^{[v]})|+\eta)\). Given that the inner Conditional SMACOF step returns \(\mathbf X^{[v+1]}\) satisfying \(\sum_k q_k^{[v]}r_k(\mathbf X^{[v+1]})^2 \le \sum_k q_k^{[v]}r_k(\mathbf X^{[v]})^2\), then   \(0 \le\sigma_\eta(\mathbf X^{[v+1]})\le \sigma_\eta(\mathbf X^{[v]})\). The sequence \(\{\sigma_\eta(\mathbf X^{[v]})\}\) is monotone nonincreasing and bounded below.
\end{theorem}
\begin{proof}
Theorem \eqref{thm:stabilized_majorization} immediately gives
\[
\sigma_\eta(\mathbf{X}^{[v+1]})\leq Q_\eta(\mathbf{X}^{[v+1]}|\mathbf{X}^{[v]})=\sigma_\eta(\mathbf{X}^{[v]})+\frac{1}{2}[\sum_kq_k^{[v+1]}r_k(\mathbf{X})^2-\sum_k q_k^{[v]}r_k(\mathbf{X})^2].
\]
Therefore, 
\[
\sigma_\eta(\mathbf X^{[v+1]})-\sigma_\eta(\mathbf X^{[v]})=\frac{1}{2}[\sum_kq_k^{[v+1]}r_k(\mathbf{X})^2-\sum_k q_k^{[v]}r_k(\mathbf{X})^2]\leq 0.
\]
Moreover, $\sigma_\eta(\mathbf X^{[v+1]})\geq 0$, and therefore $\sigma_\eta(X^{[v]})\rightarrow\bar{\sigma}_\eta$ for some finite $\bar{\sigma}_\eta\geq0$. 
\end{proof}

Building on the Conditional SMACOF procedure of \citet{10371783}, we solve this fixed-weight subproblem by applying Conditional SMACOF as an inner update. At outer iteration \(v\), the robust weights \(q_k^{[v]}\) are computed from the current residuals and then held fixed. The resulting subproblem has the same form as the conditional stress minimized in conditional SMACOF, with \(q_k^{[v]}\) replacing the original pairwise weights. We refer to the resulting outer--inner procedure as Robust Conditional SMACOF (rcSMACOF).

Recall that \(\mathbf V\in\mathbb R^{n\times q}\) denote the feature matrix whose \(i\)-th row is \(\mathbf v_i^\top\). Define the matrix \(\mathbf H^{[v]}\) by
\begin{equation}
h_{k}^{[v]}
=
\begin{cases}
-q_{k}^{[v]}, & i\neq j,\\ 
\displaystyle\sum_{l\neq i}q_{il}^{[v]}, & i=j.
\end{cases}
\label{eq:H_laplacian}
\end{equation}

Let $\mathbf U^{[v,0]}=\mathbf U^{[v]},\mathbf B^{[v,0]}=\mathbf B^{[v]}.
$
At inner iteration \(w\), define \(\mathbf C^{[v,w]}\) by
\begin{equation}
c_{k}^{[v,w]}
=
\begin{cases}
-\dfrac{
q_{k}^{[v]}\delta_{k}
}{
d_{k}(\mathbf U^{[v,w]},\mathbf B^{[v,w]})
}
&\ 
d_{k}(\mathbf U^{[v,w]},\mathbf B^{[v,w]})\neq0,
\\
0
&\ 
d_{k}(\mathbf U^{[v,w]},\mathbf B^{[v,w]})=0
\end{cases}
\label{eq:C_offdiag}
\end{equation}
for $i\neq j$, and
\begin{equation}
c_{ii}^{[v,w]}
=
-\sum_{\substack{j=1,j\neq i}}^n c_{k}^{[v,w]}.
\label{eq:C_diag}
\end{equation}

The Robust Conditional SMACOF updates are
\begin{align}
\mathbf U^{[v,w+1]}
&=
\left(\mathbf H^{[v]}\right)^+
\mathbf C^{[v,w]}
\mathbf U^{[v,w]},
\label{eq:U_update}
\\
\mathbf B^{[v,w+1]}
&=
\left(\mathbf V^\top\mathbf H^{[v]}\mathbf V\right)^+
\mathbf V^\top
\mathbf C^{[v,w]}
\mathbf V
\mathbf B^{[v,w]}.
\label{eq:B_update}
\end{align}
Here \((\cdot)^+\) denotes the Moore Penrose inverse. 

\begin{proposition}[Positive semidefinite structure under robust reweighting]
\label{prop:reweighting_regularity}
Suppose the original active graph is connected and the active covariate
differences span \(\mathbb R^q\). For every outer iteration \(v\), $0<q_k^{[v]}\le \frac{w_k}{\eta}).$
Hence \(\mathbf H^{[v]}\) is the Laplacian of the same connected graph,
$\ker(\mathbf H^{[v]})=\operatorname{span}\{\mathbf1\}, \operatorname{rank}(\mathbf H^{[v]})=n-1,$
and $\mathbf V^\top\mathbf H^{[v]}\mathbf V\succ0.$  
\end{proposition}

\begin{proof}
The support statement follows immediately from \(\eta>0\). Connectivity then
gives the standard nullspace and rank properties of a weighted graph
Laplacian. For any \(\mathbf a\ne\mathbf0\),
\[
 \mathbf a^\top\mathbf V^\top\mathbf H^{[v]}\mathbf V\mathbf a
 =\sum_{(i,j)\in\mathcal E}q_{ij}^{[v]}
 \bigl[\mathbf a^\top(\mathbf v_i-\mathbf v_j)\bigr]^2>0,
\]
because the active differences span \(\mathbb R^q\). In particular, the ordinary inverse is valid in the \(\mathbf B\)-update.
\end{proof}

After the inner loop converges, set $\mathbf U^{[v+1]}=\mathbf U^{[v,w_{\mathrm{final}}]}, \mathbf B^{[v+1]}=\mathbf B^{[v,w_{\mathrm{final}}]}.$ 
The outer weights are then recomputed from the new residuals.

\begin{algorithm}[htpb!]
\caption{rcSMACOF}
\label{alg:orr_csmacof}
\begin{algorithmic}[1]
\REQUIRE Dissimilarities \(\{\delta_k\}\), weights \(\{w_k\}\), covariates \(\mathbf V\), initial values \(\mathbf U^{[0]}\), \(\mathbf B^{[0]}\), parameter \(\eta>0\), tolerance \(\epsilon_{\mathrm{outer}}\), maximum iterations \(v_{\max}\)
\ENSURE Final estimates \(\widehat{\mathbf U}\), \(\widehat{\mathbf B}\)

\STATE Set \(v=0\) and \(\mathbf X^{[0]}=(\mathbf U^{[0]},\mathbf B^{[0]})\)

\WHILE{\(v<v_{\max}\)}
    \STATE Compute residuals
    \[
    r_k^{[v]}=\delta_k-d_k(\mathbf U^{[v]},\mathbf B^{[v]}).
    \]

    \STATE Update robust weights
    \[
    q_k^{[v]}=\frac{w_k}{|r_k^{[v]}|+\eta}.
    \]

    \STATE With \(q_k^{[v]}\) fixed, find $(\mathbf U^{[v+1]},\mathbf B^{[v+1]})$
    using Conditional SMACOF.

    \STATE Set \(\mathbf X^{[v+1]}=(\mathbf U^{[v+1]},\mathbf B^{[v+1]})\)

    \IF{\(\left|\bar{\sigma}_\eta(\mathbf X^{[v+1]})-\bar{\sigma}_\eta(\mathbf X^{[v]})\right|<\epsilon_{\mathrm{outer}}\)}
        \STATE \textbf{break}
    \ENDIF

    \STATE \(v\leftarrow v+1\)
\ENDWHILE

\STATE Set \(\widehat{\mathbf U}=\mathbf U^{[v+1]}\), \(\widehat{\mathbf B}=\mathbf B^{[v+1]}\)

\RETURN \(\widehat{\mathbf U}\), \(\widehat{\mathbf B}\)

\end{algorithmic}
\end{algorithm}

By Theorem~\ref{thm:rcMDS_convergence}, if the inner Conditional SMACOF step decreases the fixed-weight stress, then the outer iteration decreases the stabilized conditional strife \(\sigma_\eta\).

For monitoring the outer robust objective, we use the normalized stabilized strife
\begin{equation}
\hat{\sigma}_\eta(\mathbf U,\mathbf B)
=
\frac{
\sum_{k}
w_{k}
\phi_\eta\!\left(
\delta_{k}-d_{k}(\mathbf U,\mathbf B)
\right)
}{
\sum_{k}
w_{k}
\phi_\eta(\delta_{k})
},
\label{eq:normalized_stabilized_strife}
\end{equation}
provided that the denominator is nonzero. For monitoring the inner fixed-weight SMACOF loop at outer iteration \(v\), we use
\begin{equation}
\hat S^{[v]}(\mathbf U,\mathbf B)
=
\frac{
\sum_{k}
q_{k}^{[v]}
\left(
\delta_{k}-d_{k}(\mathbf U,\mathbf B)
\right)^2
}{
\sum_{k}
q_{k}^{[v]}
\delta_{k}^2
}.
\label{eq:normalized_inner_stress}
\end{equation}

\section{Computation Experiment}
\label{sec:experiment}
Unless otherwise stated, both cSMACOF and rcSMACOF are run for a maximum of 10,000 iterations and are declared converged when the absolute difference between two consecutive stress values is less than \(10^{-15}\). For rcSMACOF, we perform one conditional SMACOF update between consecutive robust weight updates. The stabilization parameter is fixed at \(\eta=0.01\) in all experiments.  

Agreement between the recovered and true configurations is evaluated across ten resampling runs, with each run generated from 10 random contaminated
dissimilarity matrices. Outliers were injected by replacing a random subset of dissimilarities with profiles drawn from a Uniform distribution $\mathcal{U}(2, 3)\times d_{max}$, where $d_{max}$ is the largest dissimilarity. The percentage of corrupted dissimilarity values ranges from 0\% to 20\%. We primarily report the mean Procrustes statistic (ss)~\citep{peres2001well}, together with the mean average canonical correlation (cc). For the Procrustes comparison, the fitted configuration is optimally rotated, reflected, translated, and scaled to match the true configuration, using the Procrustes symmetric option~\citep{vegan}. The resulting Procrustes statistic ranges from 0 to 1, where 0 indicates perfect recovery and 1 indicates a complete mismatch.

\subsection{Synthetic Car-Brand Perception}
\noindent To evaluate our method under controlled conditions, we generated a synthetic brand-perception dataset with \(N=30\) brands and seven attributes \{\textit{Quality}, \textit{Safety}, \textit{Value}, \textit{Performance}, \textit{Eco}, \textit{Design}, \textit{Tech}\}~\citep{10371783}. Attribute importance were encoded via a normalized weight vector \(\mathbf{w}=(90,88,83,82,81,70,68)/562\). For each brand \(i\) and attribute \(\ell\), baseline values were sampled as $x_{i\ell}\sim\mathcal{U}(0,1)$ and perturbed with multiplicative Gaussian noise to obtain observations \(\hat{x}_{i\ell}=x_{i\ell}\{1+0.05\,\varepsilon_{i\ell}\}\), where \(\varepsilon_{i\ell}\sim\mathcal{N}(0,1)\). Weighted profiles were formed as \(\tilde{\mathbf{x}}_{i}=\mathbf{x}_{i}\odot\mathbf{w}\), and pairwise brand dissimilarities were computed using Euclidean distance \(d_{k}\). To reflect measurement uncertainty, we further consider noisy dissimilarities \(\hat{d}_{k}=d_{k}\{1+0.05\,\epsilon_{k}\}\) with \(\epsilon_{k}\sim\mathcal{N}(0,1)\). We consider three known-feature scenarios: Case 1 uses \{Quality, Safety, Value, Performance\}; Case 2 uses \{Quality, Safety, Value, Performance, Eco\}; and Case 3 uses \{Quality, Safety, Value, Performance, Eco, Design\}.

\begin{figure}[htbp]
\centering 
\begin{tikzpicture} 
\begin{axis}[
    width=0.7\PlotBaseWidth,
    height=0.6\PlotBaseWidth,
    scale only axis,
    axis lines=box,
    title={Case 1},
     title style={
        at={(axis description cs:0.5,0.97)},
        anchor=north
    },
    xmin=0, xmax=10,
    ymin=-0.1, ymax=1.1, 
    ylabel={cc}, 
    xticklabels={},
    ytick={0.2,0.4,0.6,0.8,1},    
]
\addplot[  
    mark=x, 
    error bars/y dir=both,
    error bars/y explicit, 
] table [
    x=contamination,
    y=cmds_corr_avg,
    col sep=comma,
    y error=cmds_corr_err
]{csv/car_final1.csv}; 
\addplot[ 
    mark=diamond*,  
    error bars/y dir=both,
    error bars/y explicit,
    error bars/error mark options={draw=gray}, 
    error bars/error bar style={gray},
] table [
    x=contamination,
    y=lcmds_corr_avg,
    col sep=comma,
    y error=lcmds_corr_err
]{csv/car_final1.csv}; 
\end{axis} 
\begin{axis}[
    width=0.7\PlotBaseWidth,
    height=0.6\PlotBaseWidth,
    scale only axis,
    xmin=0, xmax=10,
    ymin=-0.1, ymax=1.1,
    axis y line*=right,
    axis x line=none,
    ylabel={ss}, 
    ylabel near ticks 
]
\addplot[  
    mark=square, 
    forget plot,
    error bars/y dir=both,
    error bars/y explicit, 
] table [
    x=contamination,
    y=cmds_ss_avg,
    col sep=comma,
    y error=cmds_ss_err
]{csv/car_final1.csv};
\addplot[  
    mark=*,
    forget plot,
    error bars/y dir=both,
    error bars/y explicit,
    error bars/error mark options={draw=gray}, 
    error bars/error bar style={gray},
] table [
    x=contamination,
    y=lcmds_ss_avg,
    col sep=comma,
    y error=lcmds_ss_err
]{csv/car_final1.csv};
\end{axis}
\end{tikzpicture}
\vspace{1pt}
\begin{tikzpicture} 
\begin{axis}[
    width=0.7\PlotBaseWidth,
    height=0.6\PlotBaseWidth,
    scale only axis,
    xticklabels={},
    title={Case 2},
     title style={
        at={(axis description cs:0.5,0.97)},
        anchor=north
    },
    axis lines=box,
    xmin=0, xmax=10,
    ymin=-0.1, ymax=1.1, 
    ylabel={cc}, 
    ytick={0.2,0.4,0.6,0.8,1}, 
    legend style={at={(0.7,.7)},
        anchor=north,
        legend columns=1,
        font=\scriptsize, 
        draw=none,
         row sep=-2pt,
        /tikz/every even column/.append style={column sep=2pt}
    }    
]
\addplot[  
    mark=x, 
    error bars/y dir=both,
    error bars/y explicit, 
] table [
    x=contamination,
    y=cmds_corr_avg,
    col sep=comma,
    y error=cmds_corr_err
]{csv/car_final2.csv};
\addlegendentry{cSMACOF cc}
\addplot[ 
    mark=diamond*,  
    error bars/y dir=both,
    error bars/y explicit,
    error bars/error mark options={draw=gray}, 
    error bars/error bar style={gray},
] table [
    x=contamination,
    y=lcmds_corr_avg,
    col sep=comma,
    y error=lcmds_corr_err
]{csv/car_final2.csv};
\addlegendentry{rcSMACOF cc}   
\addlegendimage{mark=square}
\addlegendentry{cSMACOF ss}
\addlegendimage{mark=*}
\addlegendentry{rcSMACOF ss}
\end{axis} 
\begin{axis}[
    width=0.7\PlotBaseWidth,
    height=0.6\PlotBaseWidth,
    scale only axis,
    xmin=0, xmax=10,
    ymin=-0.1, ymax=1.1,
    axis y line*=right,
    axis x line=none,
    ylabel={ss}, 
    ylabel near ticks 
]
\addplot[  
    mark=square, 
    forget plot,
    error bars/y dir=both,
    error bars/y explicit, 
] table [
    x=contamination,
    y=cmds_ss_avg,
    col sep=comma,
    y error=cmds_ss_err
]{csv/car_final2.csv};
\addplot[  
    mark=*,
    forget plot,
    error bars/y dir=both,
    error bars/y explicit,
    error bars/error mark options={draw=gray}, 
    error bars/error bar style={gray},
] table [
    x=contamination,
    y=lcmds_ss_avg,
    col sep=comma,
    y error=lcmds_ss_err
]{csv/car_final2.csv};
\end{axis} 
\end{tikzpicture}
\vspace{1pt}
\begin{tikzpicture} 
\begin{axis}[
    width=0.7\PlotBaseWidth,
    height=0.6\PlotBaseWidth,
    scale only axis,
    axis lines=box,
    title={Case 3},
    title style={
        at={(axis description cs:0.5,0.97)},
        anchor=north
    },
    xmin=0, xmax=10,
    ymin=-0.1, ymax=1.1,
    xlabel={Contamination level (\%)},
    ylabel={cc}, 
    ytick={0.2,0.4,0.6,0.8,1},  
]
\addplot[  
    mark=x, 
    error bars/y dir=both,
    error bars/y explicit, 
] table [
    x=contamination,
    y=cmds_corr_avg,
    col sep=comma,
    y error=cmds_corr_err
]{csv/car_final3.csv}; 
\addplot[ 
    mark=diamond*,  
    error bars/y dir=both,
    error bars/y explicit,
    error bars/error mark options={draw=gray}, 
    error bars/error bar style={gray},
] table [
    x=contamination,
    y=lcmds_corr_avg,
    col sep=comma,
    y error=lcmds_corr_err
]{csv/car_final3.csv}; 
\end{axis} 
\begin{axis}[
    width=0.7\PlotBaseWidth,
    height=0.6\PlotBaseWidth,
    scale only axis,
    xmin=0, xmax=10,
    ymin=-0.1, ymax=1.1,
    axis y line*=right,
    axis x line=none,
    ylabel={ss}, 
    ylabel near ticks 
]
\addplot[  
    mark=square, 
    forget plot,
    error bars/y dir=both,
    error bars/y explicit, 
] table [
    x=contamination,
    y=cmds_ss_avg,
    col sep=comma,
    y error=cmds_ss_err
]{csv/car_final3.csv};
\addplot[  
    mark=*,
    forget plot,
    error bars/y dir=both,
    error bars/y explicit,
    error bars/error mark options={draw=gray}, 
    error bars/error bar style={gray},
] table [
    x=contamination,
    y=lcmds_ss_avg,
    col sep=comma,
    y error=lcmds_ss_err
]{csv/car_final3.csv};
\end{axis}
\end{tikzpicture}
\caption{Canonical correlations (cc) and Procrustes statistic (ss) between the estimated and true configurations under three different scenarios.}
\label{car}
\end{figure}
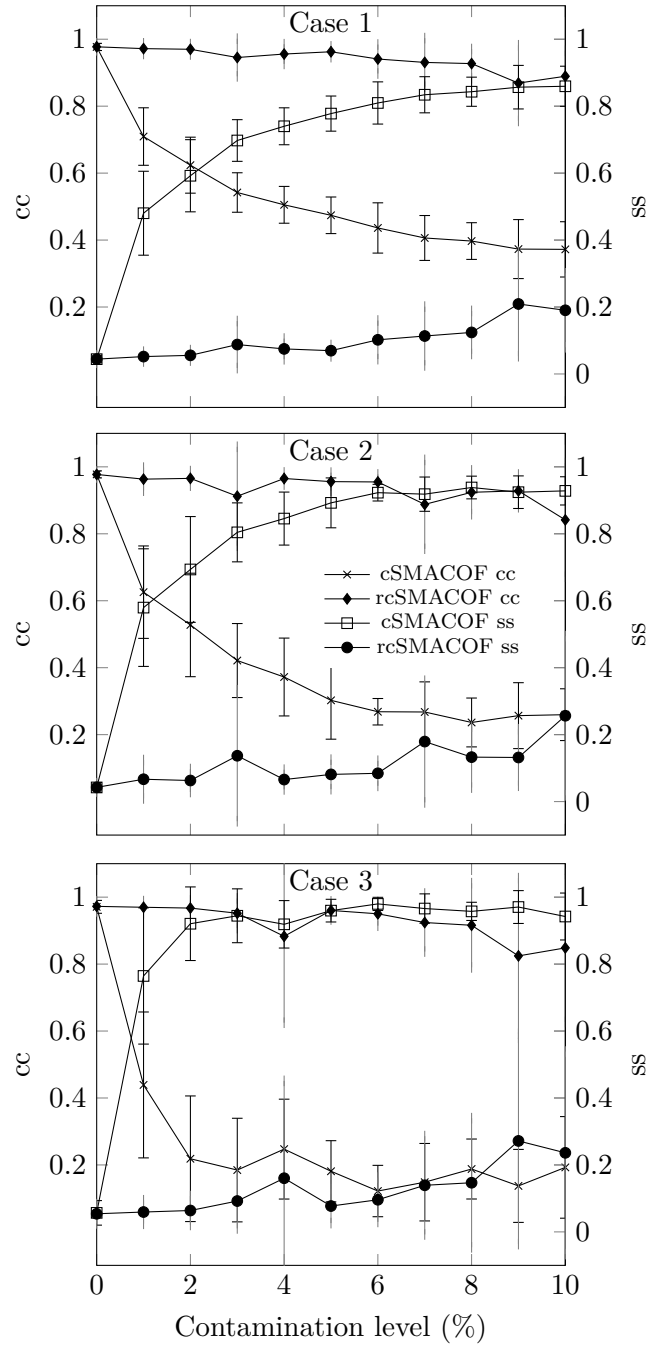
Fig.~\ref{car} shows a clear robustness advantage of rcSMACOF as the contamination level increases. For cSMACOF, the average canonical coefficient drops sharply as the contamination level increases. This indicates that the recovered configuration rapidly loses agreement with the true feature space in the presence of outlying dissimilarity values. Likewise, the Procrustes statistic increases with the contamination level, which also demonstrates the adverse effect of outliers to the solution quality. In contrast, rcSMACOF remains stable across all contamination levels: its canonical coefficient stays close to one for most contamination levels and remains above the cSMACOF curve throughout, while its Procrustes error stays close to zero with only a mild increase at higher contamination. The error bars are also consistently smaller for rcSMACOF across replications. Overall, these results suggest that the rcSMACOF effectively suppresses corrupted dissimilarities and preserves the target low-dimensional structure with high stability. rcSMACOF converged within 3,000 iterations in all cases. Detailed convergence iteration counts are provided in the supplementary material.

\subsection{Kinship Terms} 
This section demonstrates the robustness of rcSMACOF in simplifying visualization and knowledge discovery tasks. We introduced synthetic outliers into the kinship distance matrix. For each selected pair, the corresponding dissimilarity was replaced with an artificially inflated value to mimic a severe measurement anomaly. The experiments were conducted at 2\% outlier levels.  
   \begin{figure}[htbp]
\centering
\frame{%
\begin{minipage}[c][0.45\linewidth][c]{0.45\linewidth}
\centering
\begin{tikzpicture}
\begin{axis}[
  xlabel={},
  ylabel={},
  width=0.90\linewidth,
  height=0.90\linewidth,
  xtick=\empty,
  ytick=\empty,
  scale only axis,
  axis lines=none,
]
\addplot[
  only marks,
  mark=*,
  mark size=.5pt,
  point meta=explicit symbolic,
  nodes near coords,
  every node near coord/.append style={
    font=\footnotesize,
    /utils/exec={%
      \pgfkeysalso{/tikz/anchor=east}%
      \ifnum\coordindex=10\relax
        \pgfkeysalso{/tikz/anchor=south west}
      \fi
      \ifnum\coordindex=9\relax
        \pgfkeysalso{/tikz/anchor=west}
      \fi
      \ifnum\coordindex=6\relax
        \pgfkeysalso{/tikz/anchor=west}
      \fi
      \ifnum\coordindex=7\relax
        \pgfkeysalso{/tikz/anchor=north east}
      \fi
      \ifnum\coordindex=0\relax
        \pgfkeysalso{/tikz/anchor=south}
      \fi
      \ifnum\coordindex=0\relax
        \pgfkeysalso{/tikz/anchor=south west}
      \fi
    }
  },
]
table[
  x=o_cmds_x,
  y=o_cmds_y,
  meta=name,
  col sep=comma
]
{csv/kinship.csv};
\end{axis}
\node[anchor=south west] at (current axis.south west) {(a)};
\end{tikzpicture}%
\end{minipage}%
}
\hspace{2pt}
\frame{%
\begin{minipage}[c][0.45\linewidth][c]{0.45\linewidth}
\centering
\begin{tikzpicture}
\begin{axis}[
  xlabel={},
  ylabel={},
  width=0.90\linewidth,
  height=0.90\linewidth,
  xtick=\empty,
  ytick=\empty,
  scale only axis,
  axis lines=none,
]
\addplot[
  only marks,
  mark=*,
  mark size=.5pt,
  point meta=explicit symbolic,
  nodes near coords,
  every node near coord/.append style={
    font=\footnotesize,
    /utils/exec={%
      \pgfkeysalso{/tikz/anchor=east}%
      \ifnum\coordindex=11\relax
        \pgfkeysalso{/tikz/anchor=west}
      \fi
      \ifnum\coordindex=12\relax
        \pgfkeysalso{/tikz/anchor=south}
      \fi
      \ifnum\coordindex=1\relax
        \pgfkeysalso{/tikz/anchor=west}
      \fi
      \ifnum\coordindex=10\relax
        \pgfkeysalso{/tikz/anchor=west}
      \fi
      \ifnum\coordindex=9\relax
        \pgfkeysalso{/tikz/anchor=west}
      \fi
      \ifnum\coordindex=8\relax
        \pgfkeysalso{/tikz/anchor=north}
      \fi
      \ifnum\coordindex=3\relax
        \pgfkeysalso{/tikz/anchor=south}
      \fi
      \ifnum\coordindex=7\relax
        \pgfkeysalso{/tikz/anchor=west}
      \fi
      \ifnum\coordindex=6\relax
        \pgfkeysalso{/tikz/anchor=west}
      \fi
      \ifnum\coordindex=9\relax
        \pgfkeysalso{/tikz/anchor=north west}
      \fi
      \ifnum\coordindex=10\relax
        \pgfkeysalso{/tikz/anchor=south west}
      \fi
    }
  },
]
table[
  x=o_lcmds_x,
  y=o_lcmds_y,
  meta=name,
  col sep=comma
]
{csv/kinship.csv};

\draw[dotted,->] (axis cs:30,1) -- (axis cs:-30,1);
\end{axis}
\node[anchor=south west] at (current axis.south west) {(b)};
\end{tikzpicture}%
\end{minipage}%
}
\par\vspace{2pt}
\frame{%
\begin{minipage}[c][0.45\linewidth][c]{0.45\linewidth}
\centering
\begin{tikzpicture}
\begin{axis}[
  xlabel={},
  ylabel={},
  width=0.90\linewidth,
  height=0.90\linewidth,
  xtick=\empty,
  ytick=\empty,
  scale only axis,
  axis lines=none,
]
\addplot[
  only marks,
  mark=*,
  mark size=.5pt,
  point meta=explicit symbolic,
  nodes near coords,
  every node near coord/.append style={
    font=\footnotesize,
    /utils/exec={%
      \pgfkeysalso{/tikz/anchor=east}%
      \ifnum\coordindex=2\relax
        \pgfkeysalso{/tikz/anchor=west}
      \fi
      \ifnum\coordindex=8\relax
        \pgfkeysalso{/tikz/anchor=south}
      \fi
      \ifnum\coordindex=7\relax
        \pgfkeysalso{/tikz/anchor=west}
      \fi
      \ifnum\coordindex=10\relax
        \pgfkeysalso{/tikz/anchor=west}
      \fi
      \ifnum\coordindex=13\relax
        \pgfkeysalso{/tikz/anchor=south}
      \fi
      \ifnum\coordindex=9\relax
        \pgfkeysalso{/tikz/anchor=west}
      \fi
      \ifnum\coordindex=5\relax
        \pgfkeysalso{/tikz/anchor=west}
      \fi
      \ifnum\coordindex=1\relax
        \pgfkeysalso{/tikz/anchor=west}
      \fi
      \ifnum\coordindex=2\relax
        \pgfkeysalso{/tikz/anchor=south west}
      \fi
      \ifnum\coordindex=3\relax
        \pgfkeysalso{/tikz/anchor=north east}
      \fi
      \ifnum\coordindex=13\relax
        \pgfkeysalso{/tikz/anchor=north,/tikz/yshift=-8pt}
      \fi
    }
  },
]
table[
  x=o_cmds_g_x,
  y=o_cmds_g_y,
  meta=name,
  col sep=comma
]
{csv/kinship.csv};
\end{axis}
\node[anchor=south west] at (current axis.south west) {(c)};
\end{tikzpicture}%
\end{minipage}%
}
\hspace{2pt}
\frame{%
\begin{minipage}[c][0.45\linewidth][c]{0.45\linewidth}
\centering
\begin{tikzpicture}
\begin{axis}[
  xlabel={},
  ylabel={},
  width=0.90\linewidth,
  height=0.90\linewidth,
  xtick=\empty,
  ytick=\empty,
  scale only axis,
  axis lines=none,
]
\addplot[
  only marks,
  mark=*,
  mark size=.5pt,
  point meta=explicit symbolic,
  nodes near coords,
  every node near coord/.append style={
    font=\footnotesize,
    /utils/exec={%
      \pgfkeysalso{/tikz/anchor=east}%
      \ifnum\coordindex=4\relax
        \pgfkeysalso{/tikz/anchor=west}
      \fi
      \ifnum\coordindex=7\relax
        \pgfkeysalso{/tikz/anchor=south}
      \fi
      \ifnum\coordindex=5\relax
        \pgfkeysalso{/tikz/anchor=west}
      \fi
      \ifnum\coordindex=1\relax
        \pgfkeysalso{/tikz/anchor=west}
      \fi
      \ifnum\coordindex=3\relax
        \pgfkeysalso{/tikz/anchor=south}
      \fi
      \ifnum\coordindex=8\relax
        \pgfkeysalso{/tikz/anchor=south east}
      \fi
      \ifnum\coordindex=13\relax
        \pgfkeysalso{/tikz/anchor=north west}
      \fi
      \ifnum\coordindex=0\relax
        \pgfkeysalso{/tikz/anchor=north}
      \fi
      \ifnum\coordindex=12\relax
        \pgfkeysalso{/tikz/anchor=west}
      \fi
      \ifnum\coordindex=2\relax
        \pgfkeysalso{/tikz/anchor=north west}
      \fi
      \ifnum\coordindex=10\relax
        \pgfkeysalso{/tikz/anchor=north}
      \fi
      \ifnum\coordindex=0\relax
        \pgfkeysalso{/tikz/anchor=north east}
      \fi
      \ifnum\coordindex=3\relax
        \pgfkeysalso{/tikz/anchor=south,/tikz/yshift=5pt}
      \fi
      \ifnum\coordindex=8\relax
        \pgfkeysalso{/tikz/anchor=south east}
      \fi
      \ifnum\coordindex=10\relax
        \pgfkeysalso{/tikz/anchor=north east}
      \fi
      \ifnum\coordindex=12\relax
        \pgfkeysalso{/tikz/anchor=south west}
      \fi
    }
  },
]
table[
  x=o_lcmds_g_x,
  y=o_lcmds_g_y,
  meta=name,
  col sep=comma
]
{csv/kinship.csv};

\dashbigarc{-4}{2}{13}{85}{-225}
\end{axis}
\node[anchor=south west] at (current axis.south west) {(d)};
\end{tikzpicture}%
\end{minipage}%
}

\caption{Scatterplots of the recovered two-dimensional configurations for the
14 kinship terms. The first row shows the results from cSMACOF, and the second
row shows the results from rcSMACOF. The first column presents the
configurations obtained after conditioning the gender feature, while the
second column presents the configurations obtained after conditioning both
the gender and kinship-degree features. The dotted arrows indicate the
dominant direction of the remaining kinship structure in the recovered
configurations.}
\label{fig:kinship}
\end{figure}
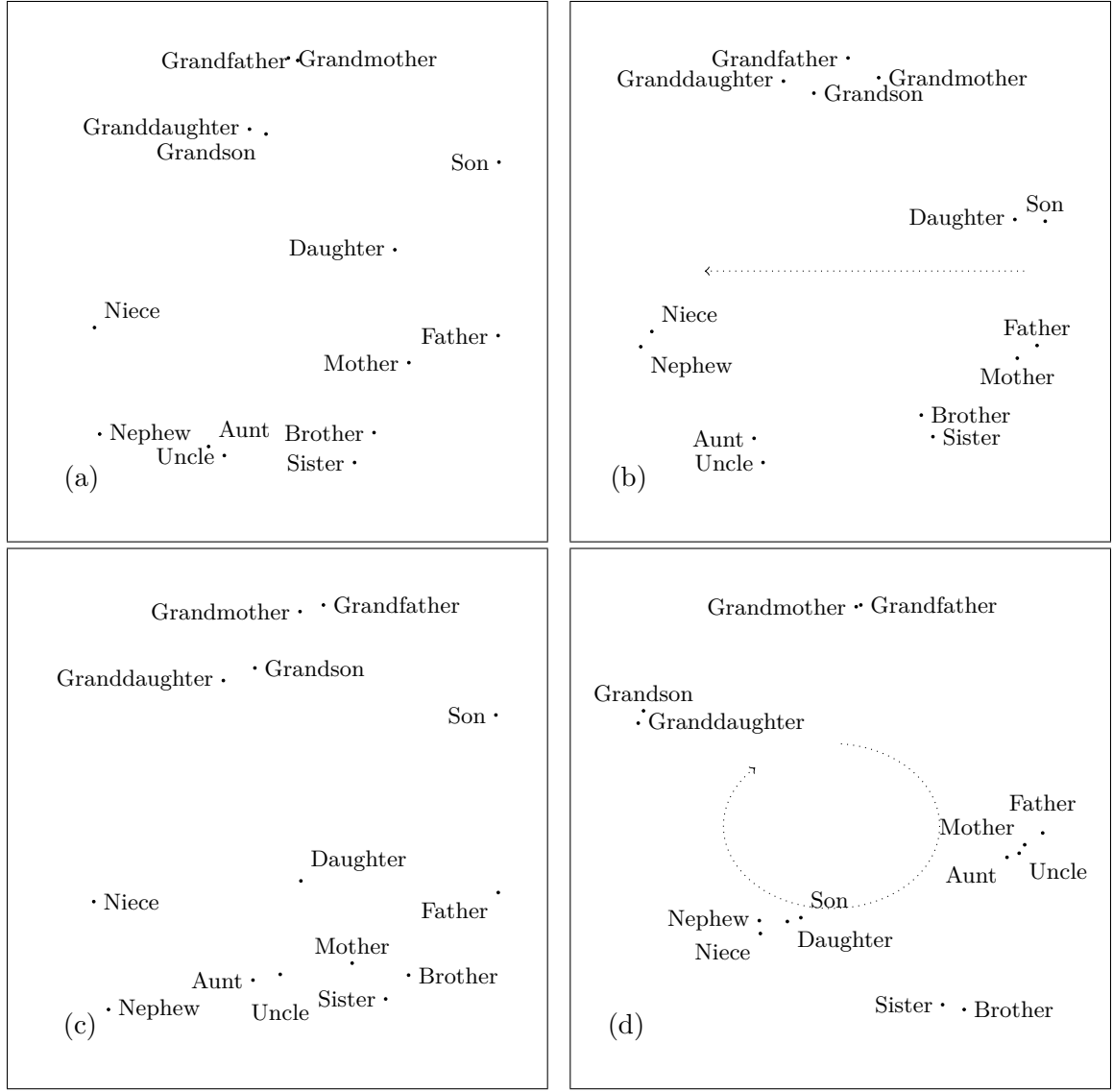
Figure \ref{fig:kinship} shows the configurations obtained by cSMACOF and rcSMACOF.  Under contaminated dissimilarities, the proposed rcSMACOF gives a more interpretable recovery of the residual kinship structure after conditioning. After conditioning the dominant gender feature, cSMACOF still shows a broad kinship-degree direction from the right side of the plot toward the left, but several paired terms are poorly preserved: for example, Daughter and Son are separated noticeably, and Niece and Nephew are also placed far apart. This suggests that the contaminated distances distort the local pairing structure in the cSMACOF configuration. By contrast, rcSMACOF more clearly reveals the hidden kinship-degree organization after gender is removed, as indicated by the dotted directional path. After conditioning both gender and kinship degree, the contrast becomes stronger. The cSMACOF configuration becomes relatively disordered, with terms appearing to jump across groups rather than preserving coherent relational neighborhoods. In the rcSMACOF configuration, however, related terms remain more consistently grouped, and the arrangement follows the remaining generation pattern indicated by the circular dotted path. These results suggest that rcSMACOF is robust to contaminated distance entries and is better able to recover the underlying residual kinship structure after removing the specified dominant features.

\subsection{Rectangles}
We used the rectangle perception data from \citet{borg1983dimensional}. The data consist of two representations of the same 16 rectangles. The first representation uses the original width and height coordinate system. The second representation uses a size and shape coordinate system, where size is defined as the sum of width and height, and shape is defined as width minus height. Each point in these grids corresponds to one rectangle, and the relative positions of the rectangles are displayed in panels (a) and (d), respectively. Outliers are introduced to the dissimilarity data in the same manner as in the previous example. For the first representation, width was used as the known feature, and height was learned as the new coordinate. For the second representation, size was used as the known feature, and shape was learned as the new coordinate.

\begin{figure}[htbp]
\centering
\frame{
\begin{tikzpicture}
\node[inner sep=0pt] (img) {\includegraphics[width=0.4\linewidth]{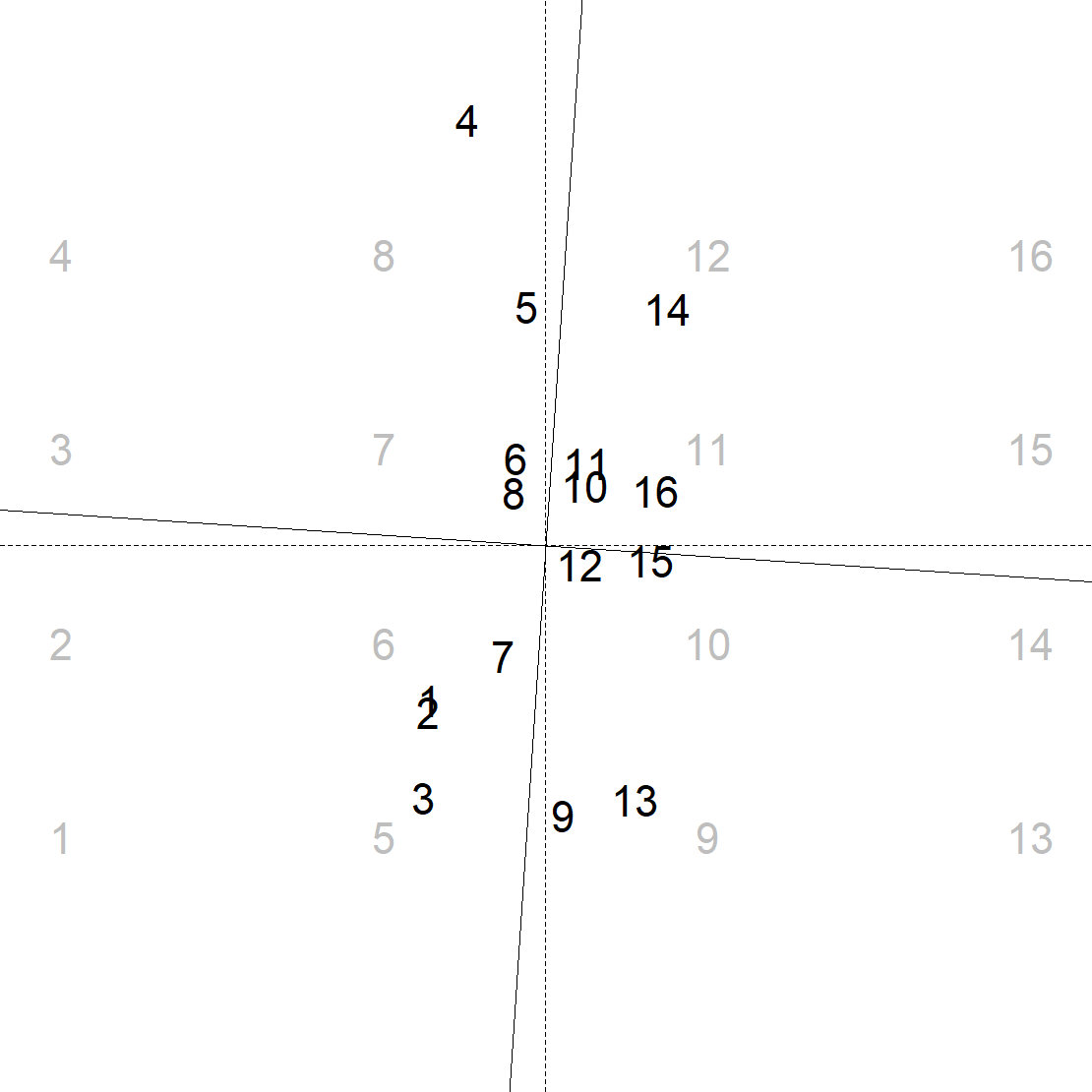}};
\node[anchor=south west] at (img.south west) {(a)};
\end{tikzpicture}
}
\vspace{2pt}
\frame{
\begin{tikzpicture}
\node[inner sep=0pt] (img) {\includegraphics[width=0.4\linewidth]{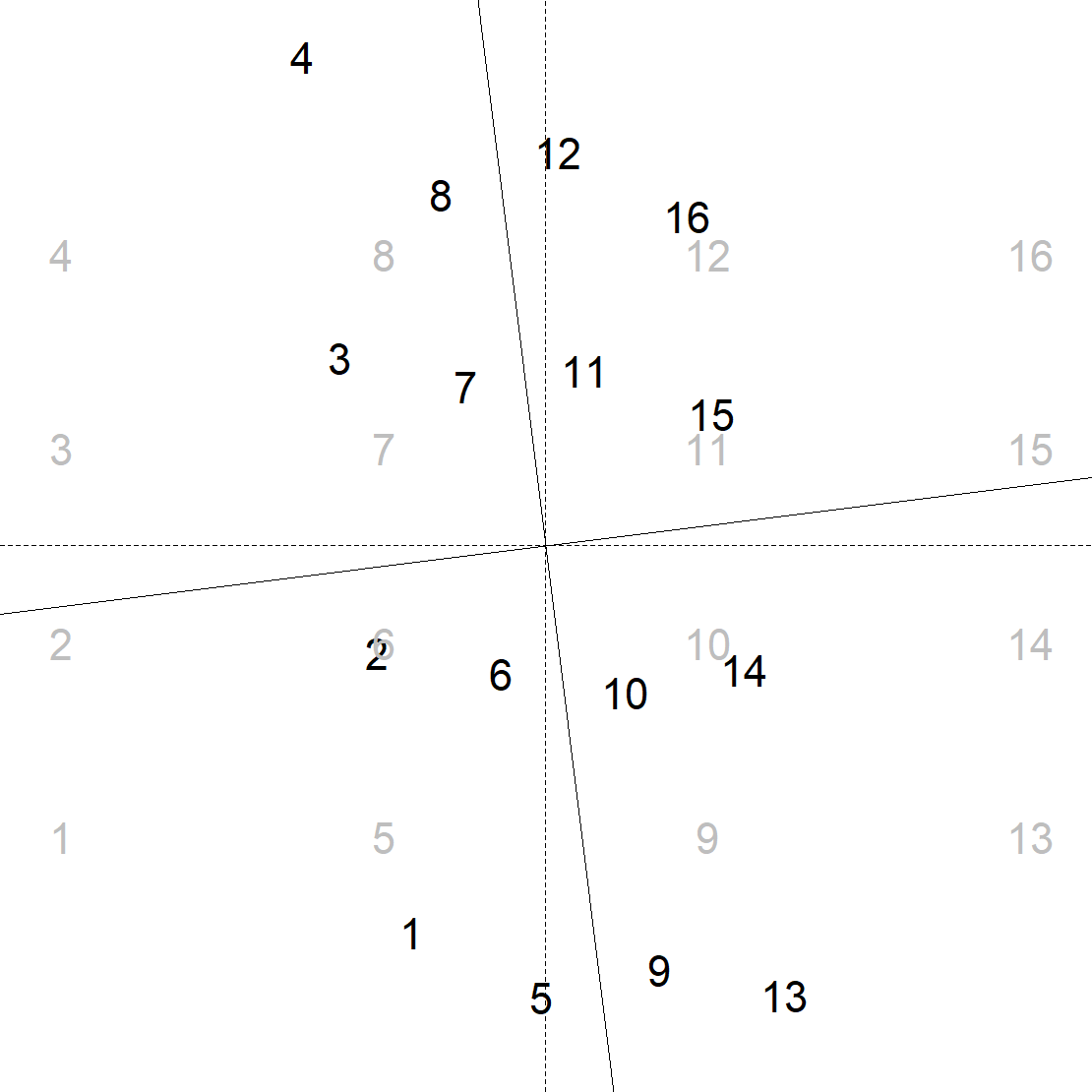}};
\node[anchor=south west] at (img.south west) {(b)};
\end{tikzpicture}
}
\vspace{2pt}
\frame{
\begin{tikzpicture}
\node[inner sep=0pt] (img) {\includegraphics[width=0.4\linewidth]{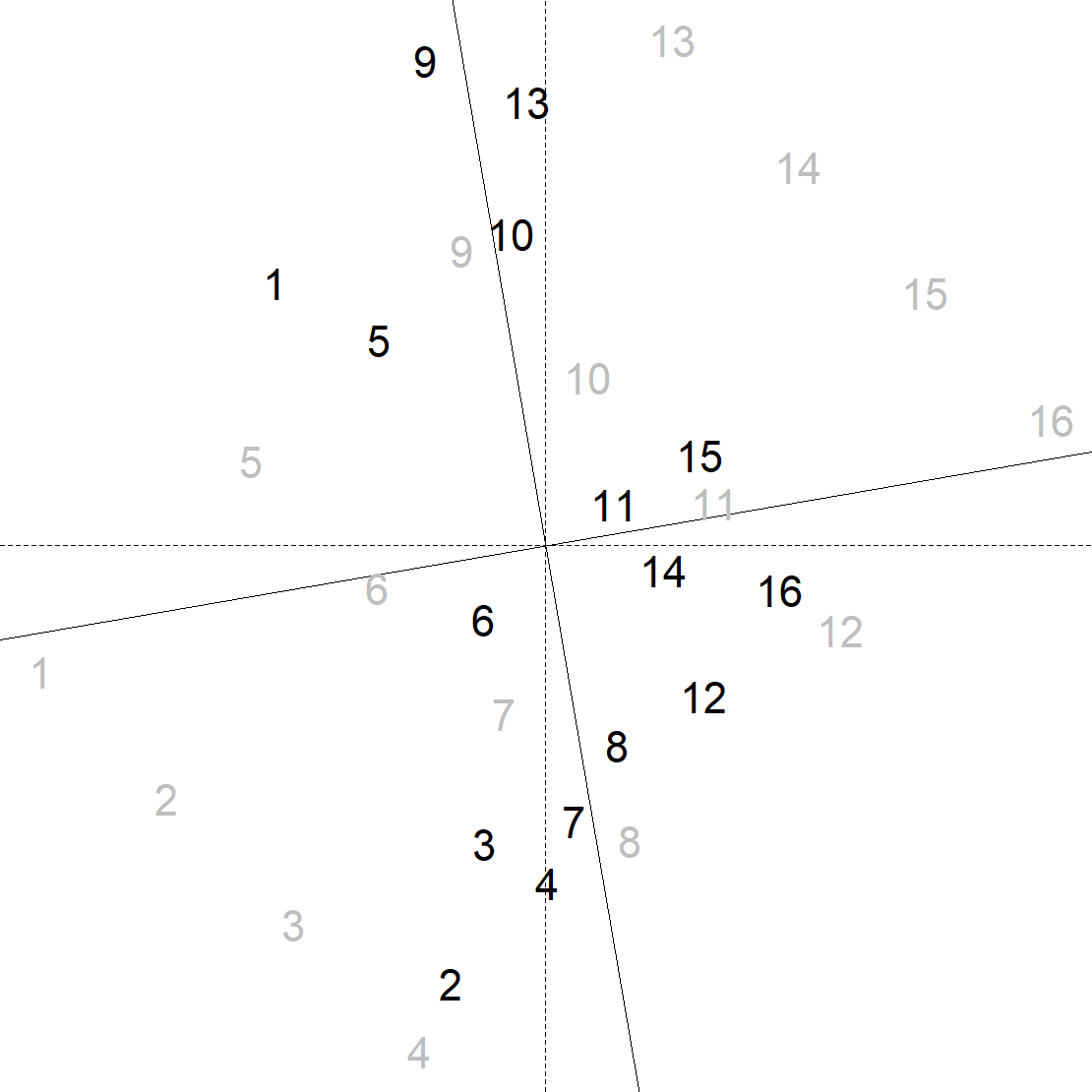}};
\node[anchor=south west] at (img.south west) {(c)};
\end{tikzpicture}
}
\vspace{2pt}
\frame{
\begin{tikzpicture}
\node[inner sep=0pt] (img) {\includegraphics[width=0.4\linewidth]{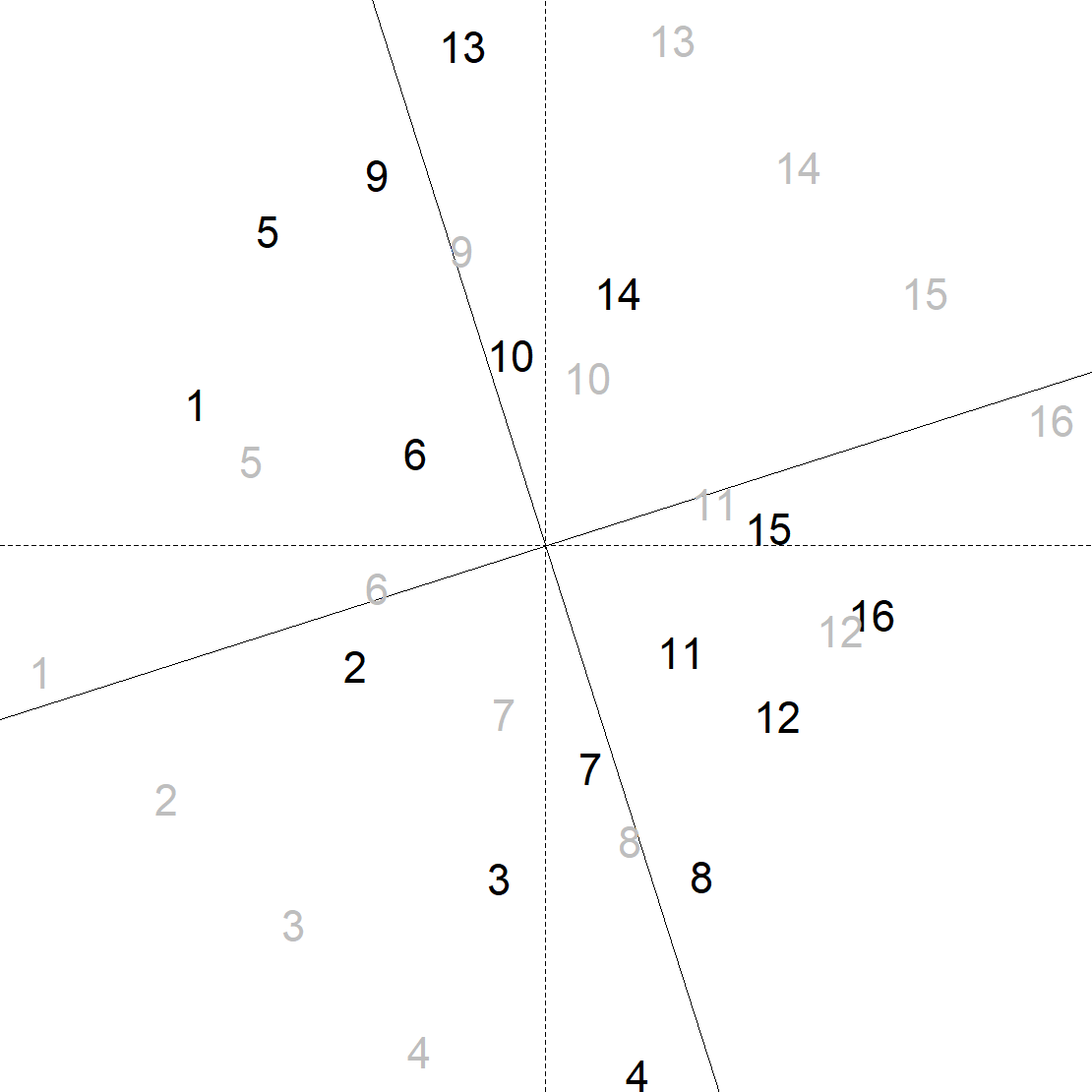}};
\node[anchor=south west] at (img.south west) {(d)};
\end{tikzpicture}
}
\caption{Design configurations in gray and recovered two-dimensional configurations in black for the rectangle data under contaminated dissimilarities. The first row shows the width–height representation: (a) cSMACOF configuration and (b) rcSMACOF configuration. The second row shows the size–shape representation: (c) cSMACOF configuration and (d) rcSMACOF configuration.}
\label{fig:figure4}
\end{figure}
 
In the width and height representation, the original design configuration in panel (a) in gray exhibits a regular grid structure. The cSMACOF configuration in black only partially preserves this organization; several rectangles are displaced from their expected neighborhoods, and the regular row and column structure is substantially distorted. In contrast, the proposed rcSMACOF configuration in panel (b) retains a much clearer approximation to the original design (up to some linear transformation), with the samples arranged in a more orderly pattern that better reflects the width and height structure. A similar pattern is observed in the size and shape representation. Configuration in gray in second row shows the transformed width and height into size and shape. Under contamination, the cSMACOF result in panel (c) is visibly unstable: several points are moved away from their natural positions, and the overall configuration no longer clearly follows the design structure. By comparison, the rcSMACOF result in panel (d) more closely follows the diagonal size and shape arrangement in the original design. The relative ordering of the rectangles is better preserved, and the configuration shows fewer severe displacements. The plot also shows the rotation between the configuration and representation necessary to make them match as closely as possible. 

\begin{table}[htbp]
\centering
\caption{Procrustes statistic (SS) between the unknown feature space and the set of known variables.}
\label{tab:rec}
\begin{tabular}{ccccc}
\hline
 & \multicolumn{2}{c}{Width and Height} & \multicolumn{2}{c}{Size and Shape} \\
\cline{2-3} \cline{4-5}
Metric & \scriptsize{cSMACOF} & \scriptsize{rcSMACOF} & \scriptsize{cSMACOF} & \scriptsize{rcSMACOF} \\
\hline
     mean&0.50 &0.01&0.47&0.02  \\
         std&0.11&0.00&0.20&0.00 \\
\hline
\end{tabular}
\end{table}

Table \ref{tab:rec} reports the nonsymmetric Procrustes error (SS) between the recovered configurations and the underlying feature space over 100 replications with independently generated outliers. The SS is computed by treating the design configuration before contamination as the target configuration, and then scaling and rotating the recovered configuration to align it as closely as possible with this target; no symmetric Procrustes transformation is used. In the width–height representation, rcSMACOF yields substantially lower error than cSMACOF, with a mean SS of 0.01 compared to 0.50, and also exhibits reduced variability (standard deviation 0.00 vs. 0.11). A similar pattern is observed in the size–shape representation, where rcSMACOF achieves a mean SS of 0.02, significantly smaller than the 0.47 obtained by cSMACOF, along with a lower standard deviation (0 vs. 0.2). These results indicate that rcSMACOF consistently provides a closer alignment to the true configuration and demonstrates greater robustness to contamination across replications.

\subsection{Face Expression}
In this section, we evaluate conditional MDS using the facial-expression dissimilarity data of \citet{abelson1962multidimensional}. The data consist of pairwise dissimilarity ratings for 13 facial expressions, obtained from 30 women who rated each pair of photographs on a nine-point scale. Abelson and Sermat showed that the three Schlosberg scales estimated by \citet{engen1958dimensional} (Pleasant--Unpleasant (PU), Attention--Rejection (AR), and Tension--Sleep (TS)) account for approximately 75\% of the variation in these dissimilarities. The absolute correlations among 3 scales are 0.18, 0.15, and 0.75 for PU--AR, PU--TS, and AR--TS, respectively.

To assess conditional recovery, we consider six known-feature scenarios. In each scenario, either one or two of the three scales \(\{\mathrm{PU},\mathrm{AR},\mathrm{TS}\}\) are provided as known features, and the remaining scale(s) are treated as unknown. For example, when PU is used as the known feature, the algorithm learns two conditional coordinates intended to recover the remaining AR--TS structure. When PU and AR are used as known features, the algorithm learns one conditional coordinate intended to recover the remaining TS structure. After fitting, we combine the known feature matrix with the learned coordinates to form the estimated full configuration. This combined configuration is then compared
with the true Schlosberg configuration \((\mathrm{PU},\mathrm{AR},\mathrm{TS})\). Recovery accuracy is evaluated using Procrustes statistic between the estimated and true configurations.

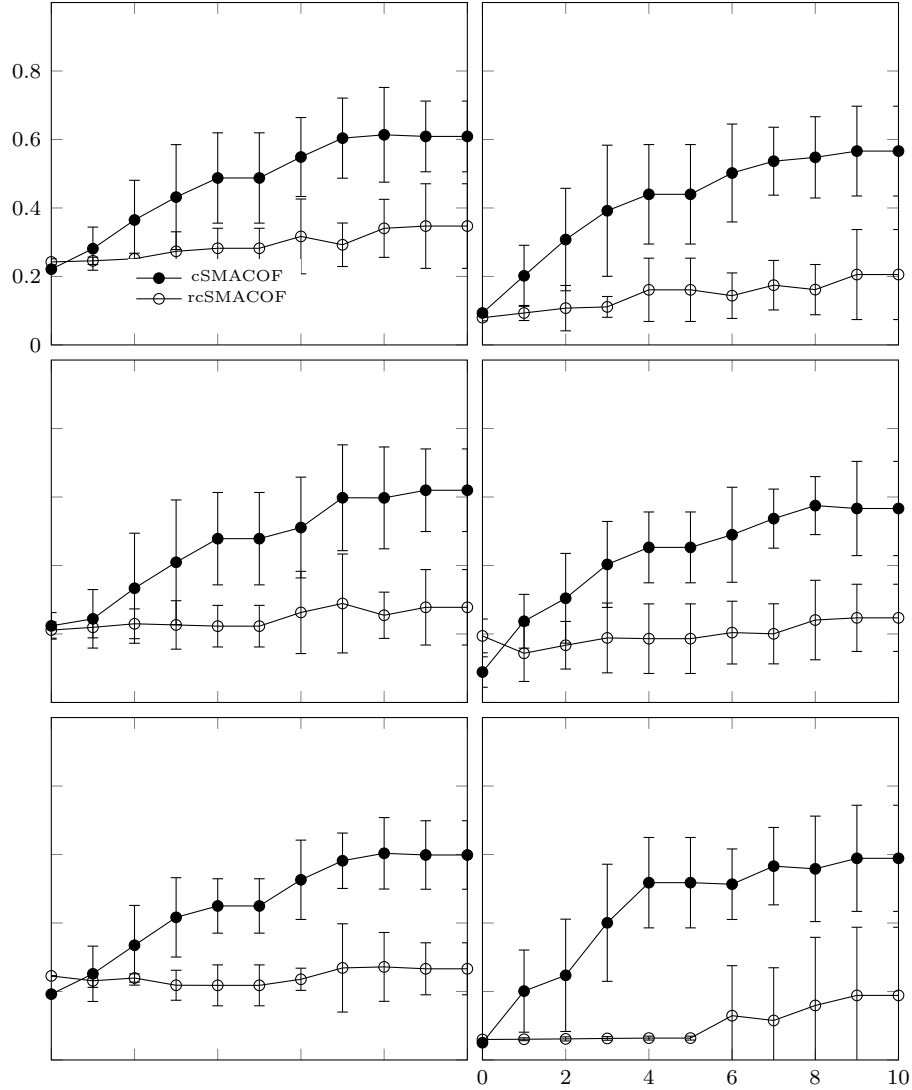
\begin{figure}[htbp]
\centering
\begin{tikzpicture}

\pgfplotsset{
    csmacofstyle/.style={
        mark=*, 
        error bars/y dir=both,
        error bars/y explicit
    },
    rcsmacofstyle/.style={
        mark=o,  
        error bars/y dir=both,
        error bars/y explicit,
        error bars/error bar style={line width=0.35pt}  
    }
}

\begin{groupplot}[
    group style={
        group size=2 by 3,
        horizontal sep=0.2cm,
        vertical sep=0.2cm
    },
    width=0.8\PlotBaseWidth, 
    xmin=0, xmax=10,
    ymin=0, ymax=1,
    axis lines=box,
    xtick={0,2,4,6,8,10},
    ytick={0,0.2,0.4,0.6,0.8},
    tick label style={font=\scriptsize},
    legend style={
        font=\tiny,
        draw=none,
        at={(0.6,0.25)},
        anchor=north east,
        row sep=-2pt
    }
]
\nextgroupplot[
    xticklabel=\empty
]
\addplot[csmacofstyle] table[
    x=contamination,
    y=pu_c,
    y error=pu_c_sd,
    col sep=comma
]{csv/face.csv};
\addlegendentry{cSMACOF}

\addplot[rcsmacofstyle] table[
    x=contamination,
    y=pu_l,
    y error=pu_l_sd,
    col sep=comma
]{csv/face.csv};
\addlegendentry{rcSMACOF}
\nextgroupplot[
    xticklabel=\empty,
    yticklabel=\empty
]
\addplot[csmacofstyle] table[
    x=contamination,
    y=ar_c,
    y error=ar_c_sd,
    col sep=comma
]{csv/face.csv};

\addplot[rcsmacofstyle] table[
    x=contamination,
    y=ar_l,
    y error=ar_l_sd,
    col sep=comma
]{csv/face.csv};
\nextgroupplot[
    xticklabel=\empty,
    yticklabel=\empty
]
\addplot[csmacofstyle] table[
    x=contamination,
    y=ts_c,
    y error=ts_c_sd,
    col sep=comma
]{csv/face.csv};

\addplot[rcsmacofstyle] table[
    x=contamination,
    y=ts_l,
    y error=ts_l_sd,
    col sep=comma
]{csv/face.csv};
\nextgroupplot[
    xticklabel=\empty,
    yticklabel=\empty
]
\addplot[csmacofstyle] table[
    x=contamination,
    y=puar_c,
    y error=puar_c_sd,
    col sep=comma
]{csv/face.csv};

\addplot[rcsmacofstyle] table[
    x=contamination,
    y=puar_l,
    y error=puar_l_sd,
    col sep=comma
]{csv/face.csv};
\nextgroupplot[
    xticklabel=\empty,
    yticklabel=\empty
]
\addplot[csmacofstyle] table[
    x=contamination,
    y=puts_c,
    y error=puts_c_sd,
    col sep=comma
]{csv/face.csv};

\addplot[rcsmacofstyle] table[
    x=contamination,
    y=puts_l,
    y error=puts_l_sd,
    col sep=comma
]{csv/face.csv};
\nextgroupplot[
    yticklabel=\empty
]
\addplot[csmacofstyle] table[
    x=contamination,
    y=arts_c,
    y error=arts_c_sd,
    col sep=comma
]{csv/face.csv};

\addplot[rcsmacofstyle] table[
    x=contamination,
    y=arts_l,
    y error=arts_l_sd,
    col sep=comma
]{csv/face.csv};
\end{groupplot}
\end{tikzpicture}
\caption{
Procrustes statistic (ss) between the estimated and true facial-expression configurations across contamination levels. Panels are ordered row-wise as follows: PU known, AR known, TS known, PU--AR known, PU--TS known, and AR--TS known. Error bars indicate standard deviations over ten replications.}
\label{fig:face_all}
\end{figure}
Fig.~\ref{fig:face_all} shows that as the contamination level increases, the Procrustes error of cSMACOF rises sharply in every panel. In contrast, rcSMACOF maintains substantially smaller Procrustes errors across all contamination levels. Averaged over the six scenarios, the cSMACOF error increases from about \(0.14\) at no contamination to about \(0.59\) at 10\% contamination, whereas the rcSMACOF error increases only from about \(0.17\) to about \(0.26\). The advantage is especially pronounced at high contamination levels, where rcSMACOF reduces the Procrustes error by roughly one half or more in all scenarios. These results indicate that the robust reweighting step effectively limits the influence of outlying dissimilarities and preserves the underlying facial-expression configuration more reliably than the standard squared-stress formulation.

\section{Discussions}
\subsection{Choice of \texorpdfstring{$\eta$}{eta}} 
To examine the sensitivity of rcSMACOF to the Fair-loss tuning parameter, we conducted an additional experiment using the simulated car-brand data. We set $\eta=\alpha s_\Delta$, where $s_\Delta$ is the median nonzero dissimilarity, and varied $\alpha$ over a logarithmic grid. For each of the three known-feature scenarios and each contamination level from 0\% to 10\%, we recorded the Procrustes statistic, and number of iterations in Figure~\ref{fig:eta_sensitivity}. The results show a clear accuracy--convergence trade-off. Very small values of $\alpha$ increase the number of iterations and often approach the maximum iteration limit. Very large values reduce the iteration count but weaken recovery. Across the three scenarios, moderate values of $\alpha$, especially around 0.01--0.1, provide the most stable balance between recovery accuracy and convergence speed. Based on this analysis, we use $\eta=0.03s_\Delta$ as the default setting in the above experiments.

\begin{figure}[htbp]
\centering
\begin{tikzpicture} 
\begin{axis}[ 
    scale only axis,
    axis lines=box,
    xmode=log,
    xmin=1e-6, xmax=3,
    ymin=0.05, ymax=0.36,
    xlabel={\(\alpha\) },
    ylabel={ss},
    xtick={1e-6,1e-5,1e-4,1e-3,1e-2,1e-1,1,3},
    xticklabels={\(10^{-6}\),\(10^{-5}\),\(10^{-4}\),\(10^{-3}\),\(10^{-2}\),\(10^{-1}\),\(1\),\(3\)},
    ytick={0.1,0.2,0.3},
    tick label style={font=\small}, 
    legend style={
        at={(0.03,0.37)},
        anchor=north west,
        font=\small,
        draw=none,
        row sep=-2pt
    }
] 
\addplot[black, solid, mark=*, line width=0.55pt]
table[x=eta_alpha, y=avg_ss_case1, col sep=comma]{csv/eta.csv};
\addlegendentry{Case 1}
\addplot[black, solid, mark=square*, line width=0.55pt]
table[x=eta_alpha, y=avg_ss_case2, col sep=comma]{csv/eta.csv};
\addlegendentry{Case 2}
\addplot[black, solid, mark=triangle*, line width=0.55pt]
table[x=eta_alpha, y=avg_ss_case3, col sep=comma]{csv/eta.csv};
\addlegendentry{Case 3}
\end{axis} 
\begin{axis}[ 
    scale only axis,
    xmode=log,
    xmin=1e-6, xmax=3,
    ymin=0, ymax=10000,
    axis y line*=right,
    axis x line=none,
    ylabel={iterations},
    ytick={0,2500,5000,7500,10000},
    tick label style={font=\small}, 
    ylabel near ticks
] 
\addplot[black, dashed, mark=*, line width=0.55pt, forget plot]
table[x=eta_alpha, y=avg_iter_case1, col sep=comma]{csv/eta.csv};
\addplot[black, dashed, mark=square*, line width=0.55pt, forget plot]
table[x=eta_alpha, y=avg_iter_case2, col sep=comma]{csv/eta.csv};
\addplot[black, dashed, mark=triangle*, line width=0.55pt, forget plot]
table[x=eta_alpha, y=avg_iter_case3, col sep=comma]{csv/eta.csv};
\end{axis}
\end{tikzpicture}
\caption{
Marker shapes distinguish the three known-feature scenarios. Solid curves report the average Procrustes statistic (ss, left axis), while dashed curves report the average number of iterations (right axis).
}
\label{fig:eta_sensitivity}
\end{figure}
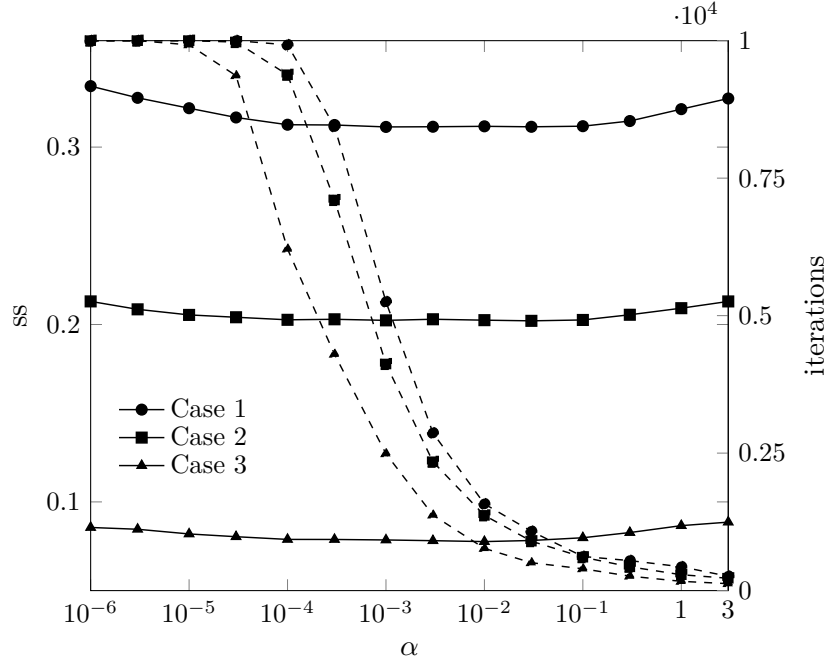

\subsection{Convergence-speed experiment} 
To evaluate computational convergence, we record the number of iterations required by cSMACOF and rcSMACOF to satisfy the stopping criterion. For each known-feature scenario and each contamination level, we generate contaminated dissimilarity matrices using the same procedure as in the accuracy experiment. Both algorithms are initialized from the same starting configuration for each replicate. cSMACOF is run until the absolute change in normalized conditional stress is below \(10^{-15}\) or until the maximum of 10,000 iterations is reached. rcSMACOF is run with the same outer stopping tolerance and \(\eta=0.01\), using one Conditional SMACOF update between consecutive robust weight updates. For each method, we record the iteration count at convergence, the final objective value, and whether the maximum iteration limit is reached.

\begin{figure}[htbp]
\centering 
\begin{tikzpicture}  
\begin{axis}[  
    axis lines=box, 
    xmin=100, xmax=900,
    ymin=0, ymax=1, 
    xlabel={Number of iterations},
    ylabel={ss}, 
    xtick={100,300,500,700,900},
    ytick={0.2,0.4,0.6,0.8,1},
    legend style={
        at={(0.97,0.57)},
        anchor=north east,
        font=\scriptsize,
        draw=none,
        row sep=-2pt
    }
]
\addplot[
    mark=diamond*,
] table [
    x=cmds10_it,
    y=cmds10_ss_avg,
    col sep=comma
]{csv/iterations.csv};
\addlegendentry{cSMACOF, 10\%}
\addplot[
    mark=diamond,
] table [
    x=lcmds10_it,
    y=lcmds10_ss_avg,
    col sep=comma
]{csv/iterations.csv};
\addlegendentry{rcSMACOF, 10\%}
\addplot[
    mark=square*,
] table [
    x=cmds5_it,
    y=cmds5_ss_avg,
    col sep=comma
]{csv/iterations.csv};
\addlegendentry{cSMACOF, 5\%}
\addplot[
    mark=square,
] table [
    x=lcmds5_it,
    y=lcmds5_ss_avg,
    col sep=comma
]{csv/iterations.csv};
\addlegendentry{rcSMACOF, 5\%} 
\end{axis}
\end{tikzpicture}
\caption{
Procrustes statistic (ss) versus the number of iterations under 5\% and 10\% contaminated dissimilarities. Lower values indicate better recovery.
}
\label{fig:iterations}
\end{figure}
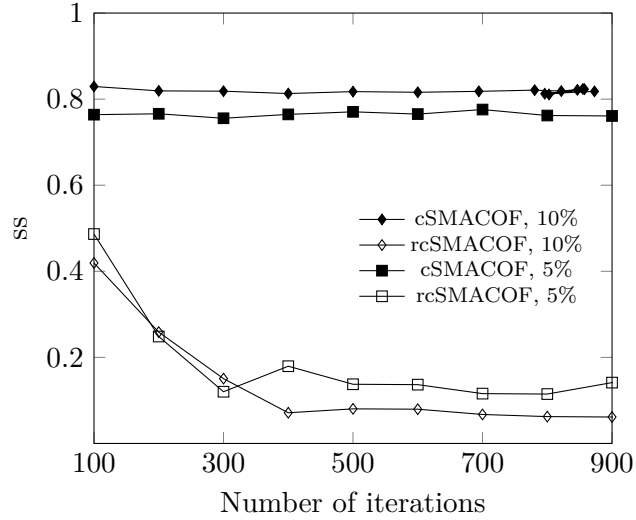

Fig.~\ref{fig:iterations} compares the Procrustes statistic of cSMACOF and rcSMACOF as a function of the number of iterations under 5\% and 10\% contaminated dissimilarities. The standard cSMACOF procedure remains sensitive to contamination: its error stays large and unchanged across iterations, indicating convergence to a poor configuration under outlier contamination. In contrast, rcSMACOF rapidly reduces the Procrustes error and stabilizes at substantially smaller errors within approximately 300 outer iterations. 

\section{Conclusion}  
\label{sec:conclusion}

This paper introduced Robust Conditional Multidimensional Scaling, a robust extension of conditional dimensional reduction for proximity data with observed covariates. The proposed formulation addresses the sensitivity of squared conditional stress to contaminated dissimilarities by replacing the squared residual loss with a Fair M-estimator. This choice preserves the robust behavior of absolute-residual fitting while avoiding the singular weights that arise in the formal \(\ell_1\) IRLS formulation. We also developed a Robust Conditional SMACOF algorithm solve the the problem.  We also established that the rcMDS is NP-hard. The numerical experiments demonstrate that rcSMACOF is robust to corrupted dissimilarities than standard cSMACOF. The study also demonstrates that small parameter $\alpha$ values can slow convergence, while overly large values reduce robustness. And we also provide a recommendation of the $\alpha$.

Overall, rcMDS provides a robust and computationally tractable framework for conditional dimension reduction from proximity data. It retains the interpretability and covariate-adjustment advantages of conditional MDS while improving reliability in the presence of corrupted pairwise dissimilarities. Future work may consider extensions to nonmetric dissimilarity transformations, and scalable implementations for large proximity matrices. 
 
\bibliographystyle{plainnat}
\bibliography{ref}

\end{document}